\documentclass{article}

\usepackage[utf8]{inputenc} 
\usepackage[T1]{fontenc}    
\usepackage{hyperref}       
\usepackage{url}            
\usepackage{booktabs}       
\usepackage{amsfonts}       
\usepackage{nicefrac}       
\usepackage{microtype}      
\usepackage{xcolor}         

\usepackage{soul}
\usepackage{times}
\usepackage{amsmath, nccmath, amssymb}
\usepackage{amsthm}
\usepackage{geometry}
\usepackage{comment}
\usepackage{enumitem}
\usepackage{color}
\usepackage{graphicx}
\usepackage{bbold}
\usepackage{xfrac}
\usepackage{thmtools}
\usepackage[ruled]{algorithm2e}
\usepackage{algorithmic}
\usepackage{mathtools}
\usepackage{mathrsfs}  
\usepackage{caption}

\usepackage{ulem}

\usepackage{graphicx}
\graphicspath{ {./images/} }
\usepackage{notation}
 
\newtheorem{theorem}{Theorem}
\newtheorem{lemma}{Lemma}
\newtheorem{remark}{Remark}

\begin{document}
\title{Gradient-free stochastic optimization for additive models} 

\author{Arya Akhavan arya.akhavan@stats.ox.ac.uk \\
   University of Oxford\\
    Alexandre B. Tsybakov alexandre.tsybakov@ensae.fr \\
      CREST, ENSAE, IP Paris
      }

\date{}

\newcommand{\problem}[1]{{\color{red} /!$\backslash$ : #1}}
\newcommand{\Mahdi}[1]{{\color{blue} EM:#1}}
\newcommand{\Arya}[1]{\color{red} AA:#1}

\maketitle

\begin{abstract}
    We address the problem of zero-order optimization from noisy observations for an objective function satisfying the Polyak-Łojasiewicz or the strong convexity condition. Additionally, we assume that the objective function has an additive structure and satisfies a higher-order smoothness property, characterized by the H\"older family of functions. The additive model for H\"older classes of functions is well-studied in the literature on nonparametric function estimation, where it is shown that such a model benefits from a substantial improvement of the estimation accuracy compared to the H\"older model without additive structure.  We study this established framework in the context of gradient-free optimization. We propose a randomized gradient estimator that, when plugged into a gradient descent algorithm, allows one to achieve minimax optimal optimization error of the order $dT^{-(\beta-1)/\beta}$, where $d$ is the dimension of the problem, $T$ is the number of queries and $\beta\ge 2$ is the H\"older degree of smoothness. We conclude that, in contrast to nonparametric estimation problems, no substantial gain of accuracy can be achieved when using additive models in gradient-free optimization.  
\end{abstract}

\section{Introduction}
Additive modeling is a popular approach to dimension reduction in nonparametric estimation problems~\cite{Stone1985AdditiveRA,hastie1986generalized,wood2017generalized}. It consists of considering that the unknown function $f:\mathbb{R}^d\to\mathbb{R}$ to be estimated from the data has the form $f(\bx)=\sum_{j=1}^d f_j(x_j)$, where $x_j$'s are the coordinates of $\bx\in \mathbb{R}^d$ and $f_j$'s are unknown functions of one variable. The main property proved in the literature on additive models in nonparametric regression can be summarized as follows. If each of the functions $f_j$ is $\beta$-H\"older (see Definition~\ref{hh} below) then the minimax rate of estimation of $f$, pointwise or in $L_2$-norm, is of the order $n^{-\beta/(2\beta+1)}$, where $n$ is the number of observations \cite{Stone1985AdditiveRA}. This is in contrast with the problem of estimating  
$\beta$-H\"older functions on $\mathbb{R}^d$ without any additive structure, since for such functions the minimax rate is known to be $n^{-\beta/(2\beta+d)}$ \cite{Stone1980OptimalRO,Stone1982OptimalGR,ibragimovkhasm1981book,Ibragimov1982BoundsFT}. Thus, there is a substantial improvement in the rate of estimation when passing from general to additive models in nonparametric regression setting.

In the present paper, we show that such a dimension reduction property fails to hold in the context of gradient-free optimization. We consider additive modeling in the problem of minimizing an unknown function $f:\mathbb{R}^d\to\mathbb{R}$ when only sequential evaluations of values of $f$ are available, corrupted with noise. We assume that $f$ is either strongly convex or satisfies the Polyak-Łojasiewicz (PL) condition \cite{polyak1963gradient,lojasiewicz1963topological} and admits an additive representation as described above, where the components $f_j$ are $\beta$-H\"older. Functions satisfying the PL condition will be for brevity referred to as PL-functions. 

The setting that we consider belongs to the family of gradient-free (or zero-order) stochastic optimization problems, for which a rich literature is now available, see \cite{Kiefer1952StochasticEO,Fabian,NY1983,PT90,jamieson2012,Shamir13,Ghadimi2013,nesterov2017random,BP2016,akhavan2020,balasubramanian2021zeroth,akhavan2023gradient,gasnikov_gradient-free_2016,Gasnikov} and the references therein. These papers did not assume any additive structure of $f$. 
It was proved in \cite{PT90} that the minimax optimal rate of the optimization error, when $f$ is $\beta$-H\"older with $\beta\ge2$ and satisfies the quadratic growth condition, is of the order $T^{-(\beta-1)/\beta}$ as function of the number of sequential queries $T$, to within an unspecified factor depending on the dimension $d$. Further developments were devoted to exploring the dependency of the minimax rate on $d$ assuming that $f$ is $\beta$-H\"older with $\beta\ge2$ and is either strongly convex \cite{jamieson2012,Shamir13,BP2016,akhavan2021distributed,akhavan2023gradient,Gasnikov} or satisfies the PL condition \cite{akhavan2023gradient,Gasnikov2024HighlySZ}. In the PL case, unconstrained minimization was studied while the strongly convex case was analyzed both in constrained and unconstrained settings. A considerable progress was achieved though the complete solution valid for all $\beta\ge2$ is not yet available.  There exists a minimax lower bound for the class of $\beta$-H\"older and strongly convex functions, which scales as $dT^{-(\beta-1)/\beta}$ (proved in \cite{Shamir13} for $\beta=2$ and Gaussian noise and in \cite{akhavan2020} for all $\beta\ge 2$ and more general noise; see also \cite{akhavan2023gradient} for a yet more general lower bound). Moreover, there exists an algorithm attaining the same rate for $\beta=2$ under general conditions on the noise (no independence or zero-mean assumption), see \cite{akhavan2020}. Thus, for $\beta=2$ we know that the minimax rate is of the order $d/\sqrt{T}$. For $\beta>2$, the literature provides different dependencies of the upper bounds on $d$ determined by the geometry of the $\beta$-H\"older condition. Thus, for the H\"older classes defined by Taylor polynomial approximations the best known upper bound is of the order $d^{2-1/\beta}T^{-(\beta-1)/\beta}$ when strongly convex functions are considered~\cite{akhavan2021distributed,Gasnikov}.  On the other hand, for the H\"older classes defined by tensor-type conditions one can achieve the rate $d^{2-2/\beta}T^{-(\beta-1)/\beta}$ both for strongly convex functions and for PL functions \cite{akhavan2023gradient}. Finally, the recent paper \cite{yu2024stochastic}, assuming again strong convexity, deals with the class of functions that admit a Lipschitz Hessian. This represents a type of H\"older condition for $\beta=3$ and the paper \cite{yu2024stochastic} derives the upper rate~$dT^{-2/3}$. We note that the lower bound of \cite{akhavan2020} with the rate $dT^{-(\beta-1)/\beta}$ is valid for all the above mentioned H\"older classes since this lower bound is obtained on additive functions that belong to all of them. Thus, the rate $dT^{-(\beta-1)/\beta}$ appears to be minimax optimal not only for $\beta=2$ but also for $\beta=3$ under a suitable definition of $3$-H\"older class of strongly convex functions.

The main result of the present paper is to establish an upper bound with the rate $dT^{-(\beta-1)/\beta}$ for the optimization error in the noisy gradient-free setting over the class of $\beta$-H\"older functions satisfying the additive model and either the PL condition or the strong convexity condition. Together with the lower bound proved in~\cite{akhavan2020}, it implies that $dT^{-(\beta-1)/\beta}$ is the minimax optimal rate of the optimization error in this setting for all $\beta\ge 2$. This conclusion is quite surprising as it goes against the intuition acquired from the classical results on nonparametric estimation mentioned above. Indeed, it means that, at least for $\beta\in \{2,3\}$, there is no improvement in the rate neither in $T$ nor in  $d$ when passing from the general $\beta$-H\"older model to the additive $\beta$-H\"older model. If any, an improvement can be obtained only in a factor independent of~$T$ and~$d$. Such a property may be explained by the fact that the optimization setting is easier than nonparametric function estimation in the sense that it aims at estimating a specific functional of the unknown $f$ (namely, its minimizer) rather than $f$ as a whole object. We also refer to a another somewhat similar fact in the gradient-free stochastic optimization setting. Specifically, there is no dramatic difference between the complexity of minimizing a strongly convex function with Lipschitz gradient, which corresponds to the case $\beta
=2$ discussed above with the minimax rate $d/\sqrt{T}$, and the complexity of minimizing a convex function with no additional properties, for which one can construct an algorithm converging with the rate between $d^{1.5}/\sqrt{T}$ and $d^{1.75}/\sqrt{T}$ (up to a logarithmic factor) as recently shown in \cite{fokkema2024online}.


\section{Problem setup}
\label{sec:problem setup}

{\color{black}Let $\com$ be a closed convex subset of $\mathbb{R}^d$.} We consider the problem of minimizing 
an unknown function 
$f:\mathbb{R}^d\to\mathbb{R}$ over the set $\com$ based on noisy evaluations of $f$ at query points that can be chosen sequentially depending on the past observations.   
Specifically, we assume that
at round 
$t\in\{1,\dots,T\}$ we observe two noisy evaluations of 
$f$ at points 
$\bz_t,\bz'_t\in\mathbb{R}^d$, i.e.,
\begin{align*}
    y_t = f(\bz_t) + \xi_t,\quad\quad y'_t = f(\bz'_t) + \xi'_t,
\end{align*}
where $\xi_t,\xi'_t$ are scalar noise variables and the query points $\bz_t,\bz'_t$ can be chosen depending on $\{\bz_i,\bz'_i,y_i,y'_i\}_{i=1}^{t-1}$ and on a possible randomization.

Throughout the paper we assume that $f$ follows the additive model 
$$f(\bx)=\sum_{j=1}^d f_j(x_j),$$ 
where $x_j$'s are the coordinates of $\bx\in \mathbb{R}^d$ and $f_j$'s are unknown functions of one variable.

We assume that each of the functions $f_j:\mathbb{R}\to \mathbb{R}$, $j=1,\dots,d$, belongs to the class of $\beta$-H\"older functions specified by the following definition, with $\beta\ge2$.

\begin{definition}\label{hh}
For $\beta>0$, $L>0$, denote by ${\cal F_\beta}(L)$ the set of all functions $f:\mathbb{R}\to \mathbb{R}$ that are $\ell=\lfloor \beta\rfloor$ times differentiable and satisfy, for all $x,z \in \mathbb{R}$, the condition
\begin{equation}
\bigg|f(z)-\sum_{m=0}^{\ell} \frac{1}{m!} f^{(m)}(x)(z-x)^{m} \bigg|\leq L |z-x|^{\beta},
\label{eq:Hclass}
\end{equation}
where  $f^{(m)}$ is the $m$-th derivative of $f$ and $\lfloor \beta\rfloor$ is the largest integer less than $\beta$. Elements of the class ${\cal F_\beta}(L)$ are referred to as {\bf $\beta$-H\"older functions}.
\end{definition}

If $\beta>2$ the fact that $f_j\in {\cal F_\beta}(L)$ does not imply that $f_j\in {\cal F}_2(L)$, however we will need the latter condition as well. It will be convenient to use it in a slightly different form given by the next definition. 

\begin{definition}\label{deflip}
The function $f:\mathbb{R}\to \mathbb{R}$ is called ${\bar L}$-smooth if it is differentiable on $\mathbb{R}$ and there exists $\bar L>0$ such that, 
for every $x,x' \in \mathbb{R}$, it holds that
\[
|f'(x)-f'(x')|\leq {\bar L} |x-x'|.
\]
The class of all ${\bar L}$-smooth functions will be denoted by ${\cal F}'_2(\bar L)$.
\end{definition}
We also assume that $f$ is either an $\alpha$-strongly convex function or an $\alpha$-PL function as stated in the next two definitions.   
\begin{definition}\label{def:PL}
Let $\alpha>0$. The function $f:\mathbb{R}^{d}\to \mathbb{R}$ is called an $\alpha$-Polyak-Łojasiewicz function (shortly, $\alpha$-PL function) if $f$ is differentiable on $\mathbb{R}^{d}$ and
\[
2\alpha\left(f(\bx) - \min_{\bz\in\mathbb{R}^d}f(\bz)\right)\leq \norm{\nabla f(\bx)}^2\quad\quad\text{for all  } \bx\in \mathbb{R}^{d},
\]
where $\|\cdot\|$ denotes the Euclidean norm.
\end{definition}

Functions satisfying the PL condition are not necessarily convex. The PL condition is a useful tool in optimization problems since it
leads to linear convergence of the gradient descent algorithm without convexity as shown by Polyak \cite{polyak1963gradient}. For more details and discussion on the PL condition see \cite{10.1007/978-3-319-46128-1_50}.

\begin{definition}\label{def:strongconv}
    Let $\alpha>0$. The function $f:\mathbb{R}^{d}\to \mathbb{R}$ is called $\alpha$-strongly convex if $f$ is differentiable on $\mathbb{R}^{d}$ and
\[
f(\bx) - f(\bx')\leq \langle\nabla f(\bx),\bx-\bx'\rangle-\frac{\alpha}{2}\norm{\bx-\bx'}^2\quad\quad\text{for all  } \bx,\bx'\in \mathbb{R}^{d}.
\]
\end{definition}

\begin{algorithm}[t!]
    \DontPrintSemicolon
\caption{}
\label{alg:ZO_SGD}
   \SetKwInput{Input}{Input}
   \SetKwInOut{Output}{Output}
   \SetKwInput{Initialization}{Initialization}
   \Input{Constraint set $\com$, function $K:[-1,1]\to\mathbb{R}$, step size $\eta_t>0$, and perturbation parameter $h_t>0$, for $t=1,\dots,T$}
   \Initialization{Generate vectors $\br_t = (r_{t,1},\dots, r_{t,d})\in \mathbb{R}^d$ for $t = 1,\dots, T$, where the components $r_{t,i}$ are independent and distributed uniformly from the interval $[-1,1]$, and choose $\bx_1\in\com$}
\For{$t = 1,\dots, T$}{
Obtain observations $y_t = f(\bx_t + h_t\br_t)+\xi_t$ and $y_t' = f(\bx_t - h_t\br_t)+\xi_t'$ \linebreak 
\For{$j = 1,\dots, d$}{
$g_{t,j}: = \frac{d}{2h_t}(y_t - y_t')K(r_{t,j})$}
Let $\bg_t: = (g_{t,1},\dots, g_{t,d})$\tcp*{gradient estimator}
$\bx_{t+1} = \proj_{\com}(\bx_t - \eta_t\bg_t)$\tcp*{update}
}
\end{algorithm}

In order to minimize $f$, we apply a version of projected gradient descent presented in Algorithm \ref{alg:ZO_SGD}. Let $\{\eta_t\}_{t=1}^T$ be a sequence of positive numbers and let $\{\bg_t\}_{t=1}^T$ be a sequence of random vectors. Consider any fixed $\bx_1\in\mathbb{R}^d$ and let the vectors $\bx_{t}$ for $t=2,\dots, T$ be defined by the recursion   
   {\color{black}
\begin{align}\label{eq:protocol}
 \bx_{t+1} = \proj_{\com}\left(\bx_t -\eta_t\bg_{t}\right),
\end{align}
where $\proj_{\com}(\cdot)$ is the Euclidean projection over $\com$.}

In this paper, the gradient estimator 
$\bg_t = (g_{t,1},\dots,g_{t,d})\in\mathbb{R}^d$  at round $t\in \{1,\dots, T\}$ of the algorithm is defined as follows.
 For given $\beta\ge 2$ and $\ell=\lfloor \beta\rfloor$, let $K:[-1,1]\to\mathbb{R}$ be a function such that 
 \begin{equation}\label{ass:K}
\int u K(u) du =1,~\int u^j K(u) du =0, \ j=0,2,3,\dots, \ell,~\text{and}~
\kappa_\beta\equiv\int |u|^{\beta} |K(u)| du  <\infty.
\end{equation}
Assume that $\kappa: = \int K^2(r)dr$ is finite. Functions $K$ satisfying these conditions are not hard to construct. In particular, a construction based on Legendre polynomials can be used, see, for example, \cite{PT90,Tsybakov09,BP2016}.

At each round $t$ of the algorithm, we generate
{a random vector $\br_t = (r_{t,1},\dots,r_{t,d}) \in\mathbb{R}^d$, where the components $r_{t,j}$ are independent and distributed uniformly on~$[-1,1]$.} For $h_t>0$, we draw two noisy evaluations
\begin{align*}
    y_t  = f(\bx_t + h_t\br_t)+\xi_t,\quad\quad y'_t  = f(\bx_t - h_t\br_t)+\xi'_t,
\end{align*}
and, for $j\in\{1,\dots,d\}$, we define 
\begin{align}\label{eq:gradest}
    g_{t,j} = \frac{1}{2h_t}(y_t - y'_t){K(r_{t,j})}.
\end{align}
{\color{black} We consider the gradient estimator $\bg_t = (g_{t,1},\dots, g_{t,d})$.} { Note that other choices of gradient estimator can lead to similar results as those that we obtain below, namely estimators based on finite difference approximations taking into account higher order smoothness. In contrast to \eqref{eq:gradest}, such higher order finite difference schemes have a complicated form and require many queries per step of the algorithm.}

We assume that the noise variables $\xi_{t}, \xi'_{t}$ and the randomizing variables $r_{t,j}$ satisfy the following.

\begin{assumption} 
\label{ass:noise}
There exists $\sigma^2> 0$ such that for all $t\in\{1,\dots,T\}$ the following  holds.
\begin{itemize}   
\item[\text{(i)}]  The random variables $r_{t,j}\sim U[-1,1]$, $j=1,\dots,d$, {are independent of $\bx_t$ and conditionally independent of $\xi_{t}, \xi'_{t}$ given $\bx_t$, }
\item[\text{(ii)}] $\mathbb{E}[\xi_{t}^2] \leq \sigma^2$, $\mathbb{E}[(\xi'_{t})^2] \leq \sigma^2$.
\end{itemize}
\end{assumption}
Assumption \ref{ass:noise} (i) can be considered not as a restriction but as a part of the definition of the algorithm dealing with the choice of the randomizing variables $r_{t,j}$. Randomizations are naturally chosen independent of all other sources of randomness. For the proofs we need even a weaker property that we state here as an assumption in order to refer to it in what follows.
Note also that Assumption~\ref{ass:noise} does not require the noises $\xi_{t}, \xi'_{t}$ to have zero mean. Moreover, they can be
non-random and we do not assume independence between these noises on different rounds of the algorithm. The fact that convergence of gradient-free algorithms can be achieved under general conditions on the noise of similar type, not requiring independence and zero means, dates back to \cite{granichin1992stochastic,polyak1992arbitrary}.

\section{Statement of the results}
\label{sec:main}

In this section, we provide upper bounds on the optimization error of the algorithm defined in Section~\ref{sec:problem setup}. 
First, we assume that function $f$ represented by the additive model satisfies the $\alpha$-PL condition. Note that imposing this condition on the sum $f$ { implies} that all the components $f_j$ are PL functions. When dealing with PL functions we consider the problem of unconstrained minimization and we introduce the notation 
$$
f^* = \min_{\bx\in\mathbb{R}^d}f(\bx).
$$
In what follows, for any positive integer~$n$ we denote by
$[n]$ the set of positive integers less than or equal to $n$.
\begin{theorem}\label{thm1}
Let $\alpha >0$, $\beta\geq 2$ and $\bL,L>0$.
 For $j\in[d]$, let $f_j:\mathbb{R}\to\mathbb{R}$ be such that $f_j\in \mathcal{F}_2'(\bL)\cap\mathcal{F}_\beta(L)$. Assume that $f:\mathbb{R}^{d}\to\mathbb{R}$ has the form $f(\bx) = \sum_{j=1}^{d}f_j(x_j)$, and that $f$ is an $\alpha$-PL function. Let Assumption~\ref{ass:noise} hold, and let $\{\bx_t\}_{t=1}^{T}$ be the realization of { Algorithm~\ref{alg:ZO_SGD}} with {\color{black}$\com = \mathbb{R}^d$}, and 
 \begin{align*}
  \eta_t &= \min\Big(\frac{4}{\alpha t},\frac{1}{18\bL{\color{black}d}\kappa}\Big),\\\vspace{4mm}
  h_t &= \Big(\frac{3\bL}{\alpha }\frac{\kappa \sigma^2}{{2}L^2\kappa_{\beta}^2}\Big)^{\frac{1}{2\beta}}\,\begin{cases}
        t^{-\frac{1}{2\beta}}&\text{if $\eta_t = \frac{4}{\alpha t}$},\vspace{3mm}\\
        T^{-\frac{1}{2\beta}}&\text{if $\eta_t = \frac{1}{18\bL {\color{black}d}\kappa}$}.
    \end{cases} 
\end{align*}
Then, 
\begin{align*}
    \Exp\left[f(\bx_{T}) - f^*\right]\leq {\frac{\cst_0}{T}(f(\bx_{1}) - f^*) +   \cst\frac{d}{\alpha }\left(\left(\frac{1}{\alpha T}\right)^{\frac{\beta-1}{\beta}}+ \frac{{\color{black}d}}{T}\left(\frac{1}{\alpha T }\right)^{\frac{1}{\beta}}\right),}
\end{align*}
where $\cst_0 = \max\big(\frac{144\bL {\color{black}d}{\kappa}}{\alpha },1\big)$ and $\cst> 0$ depends only on ${L, \bL}$, $\beta$, and $\sigma^2$. 
\end{theorem}

\begin{corollary}\label{cor1}
    Under the conditions of Theorem \ref{thm1}, if {$T \geq \alpha^{\beta/2-1}d^{\beta/2}$}, then 
\begin{align*}
    \Exp\left[f(\bx_{T}) - f^*\right]\leq \frac{\cst_0}{T}(f(\bx_{1}) - f^*) +   
    \cst\frac{d}{\alpha }\left(\frac{1}{\alpha T}\right)^{\frac{\beta-1}{\beta}},
\end{align*}
{where $\cst > 0$ is a constant that does not depend on $\alpha$, $d$, and $T$.}
\end{corollary}

Note that in Corollary~\ref{cor1} the condition 
$T\ge C d^{\beta/2}$, where $C>0$ is a constant,  does not bring an additional restriction when $2\le \beta\le 3$ since the condition is weaker than giving the range of~$T$, for which the bound makes sense. Indeed, for $2\le \beta\le 3$ the bound 
is greater than a constant independent of $T,d$ if $T\le C d^{\beta/2}$. 

Next, we study the optimization error of the algorithm defined in Section~\ref{sec:problem setup} under the assumption that the objective function is $\alpha$-strongly convex. Unlike the case of PL functions, we consider here the problem of constrained minimization
$$
\min_{\bx\in\com} f(\bx),
$$
where $\com$ is a compact convex subset of $\mathbb{R}^d$.
\begin{theorem}\label{thm2}
Let $\alpha >0$, $\beta\geq 2$ and $\bL,L>0$. Let $\com$ be a compact convex subset of $\mathbb{R}^d$.
 For $j\in[d]$, let $f_j:\mathbb{R}\to\mathbb{R}$ be such that $f_j\in \mathcal{F}_2'(\bL)\cap\mathcal{F}_\beta(L)$. Assume that $f:\mathbb{R}^{d}\to\mathbb{R}$ has the form $f(\bx) = \sum_{j=1}^{d}f_j(x_j)$, and that $f$ is an $\alpha$-strongly convex function such that $\max_{\bx\in\com}\norm{\nabla f(\bx)}\leq G$. Let Assumption~\ref{ass:noise} hold, and let $\{\bx_t\}_{t=1}^{T}$ be the realization of { Algorithm~\ref{alg:ZO_SGD}} with
 $$
 \eta_t = \frac{4}{\alpha(t+1)}, \qquad h_t = \Big(\frac{3}{2t}\frac{\kappa\sigma^2}{\kappa_\beta^2L^2}\Big)^{\frac{1}{2\beta}}.
 $$
Consider the weighted estimator
\begin{align*}
    \bar{\bx}_T = \frac{2}{T(T+1)}\sum_{t=1}^{T}t\bx_t.
\end{align*} 
 Then, for any $\bx\in\com$ we have
{
\begin{align*}
    \Exp\left[f(\bar{\bx}_{T}) - f(\bx)\right]\leq  \frac{36G^2d\kappa}{\alpha T}+\cst\frac{d}{\alpha}\left(1+dT^{-\frac{2}{\beta}}\right)T^{-\frac{\beta-1}{\beta}},
\end{align*}
}
{where $\cst>0$ depends only on ${L, \bL}$, $\beta$, and $\sigma^2$.}
\end{theorem}
{
\begin{corollary}\label{cor2}
Under the conditions of Theorem \ref{thm2}, if $T \geq  d^{\beta/2}$ then 
\begin{align*}
   \Exp\Big[f(\bar{\bx}_T) - \min_{\bx\in \com}f(\bx)\Big]\leq  \cst \frac{d}{\alpha}T^{-\frac{\beta-1}{\beta}},
\end{align*}
{where $\cst > 0$ is a constant that does not depend on $\alpha$, $d$, and $T$.}
\end{corollary}
}
Similarly to Corollary~\ref{cor1}, we may note that that for $2\le \beta\le 3$ the condition 
$T\ge Cd^{\beta/2}$ in Corollary~\ref{cor2}, where $C>0$ is a constant, does not bring an additional restriction but indicates the meaningful 
range of $T$ since the bound 
is greater than a constant when $T\le Cd^{\beta/2}$. 

{ 
\begin{remark}
In view of strong convexity, Theorem \ref{thm2} and Corollary \ref{cor2} immediately imply the corresponding bounds on the estimation error $\Exp[\norm{\bar{\bx}_T - \bx^*}^2]$, where $\bx^*$ is the minimizer of $f$ on $\com$ provided it is a global minimizer. Thus, under the assumptions of Corollary \ref{cor2}, if $\nabla f(\bx^*)=0$ we have
\begin{align*}
   \Exp[\norm{\bar{\bx}_T - \bx^*}^2]\leq  2\cst \frac{d}{\alpha^2}T^{-\frac{\beta-1}{\beta}},
\end{align*}
where $\cst>0$ is the constant from Corollary \ref{cor2}. 
\end{remark}

}

It follows from Corollary~\ref{cor2} and the proofs of the lower bounds in \cite{akhavan2020,akhavan2023gradient} that, under non-restrictive conditions on the parameters of the problem, the rate $\frac{d}{\alpha}T^{-\frac{\beta-1}{\beta}}$ in Corollary~\ref{cor2}
is minimax optimal on the class of additive functions $f$ satisfying the assumptions of Theorem~\ref{thm2}. The lower bound that we need is not explicitly stated in \cite{akhavan2020,akhavan2023gradient} but follows immediately from the proofs in those papers since the lower bounds in \cite{akhavan2020,akhavan2023gradient} are obtained on additive functions. For completeness, we provide here the statement of the lower bound for additive functions based on \cite{akhavan2023gradient}.

Consider all strategies of choosing the query points as 
$\bz_t= \Phi_t\big(\left(\bz_i,y_i\right)_{i=1}^{t-1},\left(\bz'_i,y'_i\right)_{i=1}^{t-1},\tau_t\big)$ and $\bz'_t= \Phi'_t\big(\left(\bz_i,y_i\right)_{i=1}^{t-1},\left(\bz'_i,y'_i\right)_{i=1}^{t-1},\tau_t\big)$ for $t\ge 2$, where $\Phi_t$'s and $\Phi'_t$'s are measurable functions, $\bz_1, \bz'_1 \in \mathbb{R}^d$
are any random variables, and $\{\tau_t\}$ is a sequence of random variables with values in a measurable space $(\mathcal Z, \mathcal U)$, such that $\tau_t$ is independent of $\big(\left(\bz_i,y_i\right)_{i=1}^{t-1},\left(\bz'_i,y'_i\right)_{i=1}^{t-1}\big)$. We denote by $\Pi_T$ the set of all such strategies of choosing query points up to $t=T$.
The class $\Pi_T$ includes the sequential strategy of the algorithm of Section~\ref{sec:problem setup} with the 
gradient estimator~\eqref{eq:gradest}. In this case, $\tau_t=\br_t$, $\bz_t=\bx_t+h_t\br_t$ and $\bz'_t=\bx_t-h_t\br_t$. 

The lower bound of \cite{akhavan2023gradient} that we are using here is proved under the following assumption on the noises~$(\xi,\xi'_t)$. Let $H^2(\cdot, \cdot)$ be the squared Hellinger distance defined, for two probability measures $\mathbf{P}, \mathbf{P}'$ on a measurable space $(\Omega, \mathcal{A})$, as 
\[
H^2(\mathbf{P}, \mathbf{P}') \triangleq \int (\sqrt{\d \mathbf{P}} - \sqrt{\d \mathbf{P}'})^2\enspace.
\]

\begin{assumption}
\label{ass:lower-bound}
For every $t\ge 1$, the following holds:
\begin{itemize}
    \item The cumulative distribution function $F_t : \bbR^2 \to \bbR$ of random variable $(\xi_t,\xi'_t)$ 
is such that
\begin{equation}
\label{distribution}
H^2 (P_{F_t(\cdot,\cdot)}, P_{F_t(\cdot+v,\cdot+w)})\leq I_{0}\max(v^2,w^2)\,, \quad\quad |v|, |w|\leq v_{0},
\end{equation}
for some $0<I_{0}<\infty$, $0<v_{0}\leq \infty$. Here, $P_{F(\cdot,\cdot)}$ denotes the probability measure corresponding to the cumulative distribution function $F(\cdot,\cdot)$.
\item The random variable $(\xi_t,\xi'_t)$ is independent of $((\bz_i,y_i)_{i=1}^{t-1},(\bz'_i,y'_i)_{i=1}^{t-1},\tau_t)$. 
\end{itemize}
\end{assumption}
Let $\com=\{\bx\in\mathbb{R}^d\,:\, \norm{\bx}\le 1\}$.  Fix $\alpha, L, \bar L>0,G>\alpha$, $ \beta\ge 2$,  and denote by  $\mathcal{F}$ the set of all functions $f$ that satisfy the assumptions of Theorem~\ref{thm2} and attain their minimum over $\mathbb{R}^d$ in $\com$. 

\begin{theorem}
\label{thm:lb}
Let $\com=\{\bx\in\mathbb{R}^d\,:\, \norm{\bx}\le 1\}$ and let Assumption~\ref{ass:lower-bound} hold. Assume that $\alpha>T^{-1/2+1/\beta}$ and $T\ge d^{\beta}$.
{Then, for any estimator $\tilde\bx_T$ based on the observations $((\bz_t,y_t), (\bz'_t,y'_t), t=1,\dots, T)$, where $((\bz_t,\bz'_t), t=1,\dots, T)$ are obtained by any strategy in the class $\Pi_T$ we have}
\begin{equation}\label{eq1:lb}
\sup_{f \in \mathcal{F}}\Exp\big[f(\tilde\bx_T)-\min_{\bx\in \com}f(\bx)\big]\geq C\frac{d}{\alpha}T^{-\frac{\beta-1}{\beta}},
\end{equation}
 where  $C>0$ is a constant that does not depend  of $T,d$, and $\alpha$. 
\end{theorem}
Theorem \ref{thm:lb} follows immediately from the proof of Theorem~22 in~\cite{akhavan2023gradient} since the family of functions used there belongs to the class $\mathcal{F}$. The lower bound of Theorem~22 in~\cite{akhavan2023gradient} has the form 
$$
C\min\left(\max(\alpha,T^{-1/2+1/\beta}), \frac{d}{\sqrt{T}}, \frac{d}{\alpha}T^{-\frac{\beta-1}{\beta}}\right),
$$
which reduces to $C\frac{d}{\alpha}T^{-\frac{\beta-1}{\beta}}$ under the assumptions on $T,d$ and $\alpha$ used in Theorem~\ref{thm:lb}.

{ 
\begin{remark}
     Since the strong convexity and the PL property hold for each additive component of $f$, another possible approach would be to run the procedure component-wise (minimize separately each component~$f_j$). However, it leads to a worse result. Indeed, in this case we need to make at each step $2d$ queries in parallel (two queries for each component) and thus can make only $\sim T/d$ steps if  the total budget of queries is $T$. At the end, applying Theorems~\ref{thm1} or~\ref{thm2} in the one-dimensional case, for each component we obtain the error in $T$ and $d$ of the order $(T/d)^{-(\beta-1)/\beta}$. This rate cannot be improved as follows from the one-dimensional instance of Theorem~\ref{thm:lb}. Summing up over the $d$ components, the overall error will be of the order 
$d (T/d)^{-(\beta-1)/\beta} = d^{2-1/\beta} T^{-(\beta-1)/\beta}$,
that is, the error will depend on $d$ in a sub-optimal way. 
\end{remark}

}

\section{Proofs}

We start by proving some auxiliary lemmas.
\begin{lemma}\label{lem:Aux1}
Let $f:\mathbb{R}^d\to\mathbb{R}$ be a differentiable function such that $\norm{\nabla f(\bx)- \nabla f(\bx')}\le \bL \norm{\bx-\bx'}$ for all $\bx,\bx'\in \mathbb{R}^d$,   
where $\bL>0$. Let Assumption \ref{ass:noise} hold and let $\{\bx_t\}_{t=1}^{T}$ be the realization of { Algorithm~\ref{alg:ZO_SGD}} with {\color{black}$\com = \mathbb{R}^d$}. Then, for all $t\in[T]$ we have
    \begin{align}\label{eq:Auxlem1}
    \Exp\left[f(\bx_{t+1})|\bx_t\right]
     &\leq f(\bx_t) -\frac{\eta_t}{2}\norm{\nabla f(\bx_t)}^2 +\frac{\eta_t}{2}\norm{\Exp\left[\bg_t|\bx_t\right]-\nabla f(\bx_t)}^2+\frac{\bar{L} \eta_t^2}{2}\Exp\left[\norm{\bg_{t}}^2|\bx_t\right].
    \end{align}
\end{lemma}
\begin{proof}
 The assumption on $f$ implies that
\begin{align*}
    f(\bx_{t+1}) &= f(\bx_t -\eta_t\bg_t) 
  \leq f(\bx_t) -\eta_t\langle \nabla f(\bx_t), \bg_t\rangle +\frac{\bar{L} \eta_t^2}{2}\norm{\bg_t}^2.
\end{align*}
By adding and subtracting $\eta_t\norm{\nabla f(\bx_t)}^2$ we obtain 
\begin{align*}
f(\bx_{t+1})
&\leq f(\bx_t) -\eta_t\norm{\nabla f(\bx_t)}^2-\eta_t\langle \nabla f(\bx_t), \ \bg_t-\nabla f(\bx_t)\rangle +\frac{\bar{L} \eta_t^2}{2}\norm{\bg_{t}}^2.
     \end{align*}
 {\color{black}Taking the conditional expectation gives 
\begin{align*}
\Exp\left[f(\bx_{t+1})|\bx_t\right]
&\leq f(\bx_t) -\eta_t\norm{\nabla f(\bx_t)}^2-\eta_t\langle \nabla f(\bx_t), \ \Exp\left[\bg_t|\bx_t\right]-\nabla f(\bx_t)\rangle +\frac{\bar{L} \eta_t^2}{2}\Exp\left[\norm{\bg_{t}}^2|\bx_t\right]
\\&\leq f(\bx_t) -\eta_t\norm{\nabla f(\bx_t)}^2 +\eta_t \norm{\nabla f(\bx_t)}\norm{\Exp\left[\bg_t|\bx_t\right] - \nabla f(\bx_t)}+\frac{\bar{L} \eta_t^2}{2}\Exp\left[\norm{\bg_{t}}^2|\bx_t\right].
     \end{align*}
} 
The lemma follows by using the inequality $2ab\leq a^2 + b^2$, $\forall a,b\in\mathbb{R}$. 
\end{proof}

{\color{black}\begin{lemma}\label{lem:bias}
    For $j\in[d]$, let $f_j:\mathbb{R}\to\mathbb{R}$ be such that $f_j\in \mathcal{F}_\beta(L)$, where $\beta\geq 2$ and $L>0$. Assume that $f:\mathbb{R}^{d}\to\mathbb{R}$ satisfies the additive model $f(\bx) = \sum_{j=1}^{d}f_j(x_j)$. Let Assumption \ref{ass:noise}(i) hold. Then, for all $t\in[T]$ we have
    \begin{align*}
        \norm{\Exp\left[\bg_t | \bx_t\right] - \nabla f(\bx_t)}\leq \kappa_{\beta} L\sqrt{d}h_t^{\beta-1}.
    \end{align*}
\end{lemma}
\begin{proof}
    Using Assumption \ref{ass:noise}(i) for any $j\in[d]$ and $t\in[T]$ we have $\Exp(K(r_{t,j}))=0$, $\Exp\left[\xi_tK(r_{t,j})|\bx_t\right]=0$ and $\Exp\left[\xi'_tK(r_{t,j})|\bx_t\right]=0$. Thus,
\begin{align*}
  \Exp\left[g_{t,j}|\bx_t\right]&= \frac{1}{2h_t}\Exp\left[(f_j(x_{t,j} + h_t r_{t,j})-f_j(x_{t,j} - h_t r_{t,j}))K(r_{t,j})|\bx_t\right] \\
  & \qquad + \frac{1}{2h_t}\sum_{m\ne j}\Exp\left[(f_m(x_{t,m} + h_t r_{t,m})-f_m(x_{t,m} - h_t r_{t,m}))K(r_{t,j})|\bx_t\right]
  \\ 
  &= \frac{1}{2h_t}\Exp\left[(f_j(x_{t,j} + h_t r_{t,j})-f_j(x_{t,j} - h_t r_{t,j}))K(r_{t,j})|\bx_t\right], 
\end{align*}
where we have used the fact that $\Exp(K(r_{t,j}))=0$ and $r_{t,j}$ is independent of $r_{t,m},$ for $m\ne j$, and of $\bx_t$. 
From Taylor expansion we obtain that
\begin{align*}
\frac{1}{2h_t}\left(f_j(x_{t,j} + h_t r_{t,j})-f_j(x_{t,j} - h_t r_{t,j})\right) &= f'_j(x_{t,j})r_{t,j} +\frac{1}{h_t}\sum_{1\leq m \leq \ell, \,m \text{  odd}}\frac{h_t^m}{m!}f^{(m)}_j(x_{t,j})r_{t,j}^m \\&\quad\quad+ \frac{R(h_tr_{t,j}) - R(-h_tr_{t,j})}{2h_t},
\end{align*}
where $|R(-h_tr_{t,j})|,|R(h_tr_{t,j})|\leq L|r_{t,j}|^{\beta}h_t^{\beta}$. Multiplying both sides of this inequality by $K(r_{t,j})$ and taking the conditional expectation implies
\begin{align*}
    |\Exp\left[g_{t,j}|{ \bx_{t}}\right] - f'_j(x_{t,j})|\leq L\Exp\left[|r_{t,j}|^{\beta}K(r_{t,j})\right]h_t^{\beta-1} = \kappa_{\beta}Lh_{t}^{\beta-1}.
\end{align*}
{ The result of the lemma follows from this inequality and the fact that \begin{align*}
        \norm{\Exp\left[\bg_t | \bx_t\right] - \nabla f(\bx_t)}\leq \sqrt{d} \max_{j=1,\dots,d} |\Exp\left[g_{t,j}|\bx_{t}\right] - f'_j(x_{t,j})|.        
    \end{align*}
    }
\end{proof}
}

{\color{black}\begin{lemma}\label{lem:var}
    For $j\in[d]$, let $f_j:\mathbb{R}\to\mathbb{R}$ be such that $f_j\in \mathcal{F}_2'(\bL)$, where  $\bL>0$. Assume that $f:\mathbb{R}^{d}\to\mathbb{R}$ satisfies the additive model $f(\bx) = \sum_{j=1}^{d}f_j(x_j)$. Let Assumption \ref{ass:noise} hold. Then, for all $t\in[T]$ we have
  \begin{align*}
 \Exp\left[\norm{\bg_{t}}^2|\bx_t\right] &\leq \frac{3}{2}\kappa d\left(\frac{3}{4}\left(d\bL^2 h_t^2 + 8\norm{\nabla f(\bx_t)}^2\right) + \frac{\sigma^2}{h_t^2}\right).
\end{align*}
\end{lemma}
\begin{proof}
For $i\in[d]$ we define
\begin{align*}
    G_i = f_i(x_{t,i} + h_tr_{t,i}) - f_i(x_{t,i} - h_tr_{t,i}).
\end{align*}
We have
\begin{align*}
    \Exp\left[g^2_{t,j}|\bx_t\right] &=\frac{1}{4h_t^2}\Exp\left[\left(\sum_{i=1}^{d}G_i +\xi_t - \xi'_t\right)^2K^2(r_{t,j})|\bx_t\right]
\end{align*}
Note that $r_{t,i}$ and $-r_{t,i}$ have the same distribution. Therefore, $\Exp[G_i|\bx_{t}] = 0 $ and we can write
\begin{align*}
    \Exp\left[g^2_{t,j}|\bx_t\right] &\leq\frac{3}{4h_t^2}\Exp\left[\left(\left(\sum_{i=1}^{d}G_i\right)^2 +\xi_t^2 + (\xi'_t)^2)\right)K^2(r_{t,j})|\bx_t\right]
    \\& \leq \frac{3}{4h_t^2}\Exp\left[\sum_{i=1}^{d}G_i^2  K^2(r_{t,j}) + \sum_{i,k=1,i\neq k}^{d}G_iG_kK^2(r_{t,j})|\bx_t\right] + \frac{3\sigma^2\kappa}{2h_t^2}
    \\&=\frac{3}{4h_t^2}\Exp\left[\sum_{i=1}^{d}G_i^2  K^2(r_{t,j}) |\bx_t\right] + \frac{3\sigma^2\kappa}{2h_t^2}.
\end{align*}
Since $f_i\in\mathcal{F}_2'(\bL)$  we have that, for all $i\in[d]$,
\begin{align*}
    G_i^2 &= \left(\left(f_i(x_{t,i}+h_tr_{t,i}) - f(x_{t,i}) - f'_i(x_{t,i})h_tr_{t,i}\right) - \left(f_i(x_{t,i}-h_tr_{t,i}) - f(x_{t,i})+ f'_i(x_{t,i})h_tr_{t,i}\right) \right.
    \\ & \qquad \left. + \, 2f'_i(x_{t,i})h_tr_{t,i}\right)^2
    \\&\leq 3\left(\left(f_i(x_{t,i}+h_tr_{t,i}) -  f(x_{t,i})-f'_i(x_{t,i})h_tr_{t,i}\right)^2 + \left(f_i(x_{t,i}-h_tr_{t,i}) - f(x_{t,i})+ f'_i(x_{t,i})h_tr_{t,i}\right)^2 \right.
    \\ & \qquad \left.+\, 4\left(f'_i(x_{t,i})h_tr_{t,i}\right)^2\right)
    \\&\leq 3\left(\frac{\bL^2}{2}h_t^4+4\left(f'_i(x_{t,i})\right)^2h^2_t\right).
\end{align*}
Thus,
\begin{align*}
        \Exp\left[g^2_{t,j}|\bx_t\right] 
        &\leq \frac{9}{4}\kappa\left(\frac{d\bL^2}{2}h_t^2 + 4\sum_{i=1}^{d}(f'_i(x_{t,i}))^2\right)+ \frac{3\sigma^2\kappa}{2h_t^2},
\end{align*}
which implies the lemma.
\end{proof}

}

\begin{lemma}\label{lem:Aux2}
 For $j\in[d]$, let $f_j:\mathbb{R}\to\mathbb{R}$ be such that $f_j\in \mathcal{F}_2'(\bL)\cap\mathcal{F}_\beta(L)$, where $\beta\geq 2$ and $\bL,L>0$. Assume that $f:\mathbb{R}^{d}\to\mathbb{R}$ satisfies the additive model $f(\bx) = \sum_{j=1}^{d}f_j(x_j)$. Let Assumption \ref{ass:noise} hold. Assume that $\{\bx_t\}_{t=1}^{T}$ is the realization of { Algorithm \ref{alg:ZO_SGD}} with {\color{black}$\com = \mathbb{R}^d$} and with the gradient estimators defined by \eqref{eq:gradest}. Then, for all $t\in[T]$ we have
    \begin{align*}
          \Exp\left[f(\bx_{t+1})|\bx_t\right] 
\leq f(\bx_t) -\frac{\eta_t}{2}\left(1-9\bar{L}{\color{black}d}\kappa\eta_t\right)\norm{\nabla f(\bx_t)}^2+\frac{\eta_t}{2}d\left(L\kappa_{\beta}h_t^{\beta-1}\right)^2+\frac{3\bar{L} \eta_t^2}{4}\kappa d\left(\frac{\sigma^2}{h_t^2} + \frac{3\bar{L}^2{\color{black}d}}{4}h_t^2\right).
    \end{align*}
\end{lemma}

\begin{proof} The result follows by combining Lemmas~\ref{lem:Aux1}, \ref{lem:bias} and \ref{lem:var}. 
\end{proof}

\begin{lemma}\label{lem:Auxstr1}
 For $\alpha >0$ assume that $f$ is an $\alpha$-strongly convex function. Assume that $\{\bx_t\}_{t=1}^{T}$ is the realization of { Algorithm \ref{alg:ZO_SGD}}. Then, for all $t\in[T]$ and $\bx\in\com$, we have
  \begin{align*}
     f(\bx_t) -f(\bx)&\leq 
\left(2\eta_t\right)^{-1}\left(\norm{\bx_t-\bx}^2 -\Exp\left[\norm{\bx_{t+1}-\bx}^2|\bx_t\right]\right) +\frac{1}{\alpha}\norm{\nabla f(\bx_t)-\Exp\left[\bg_t|\bx_t\right]}^2\\&\quad\quad+\frac{\eta_t}{2}\Exp\left[\norm{\bg_t}^2|\bx_t\right]-\frac{\alpha}{4}\norm{\bx_t-\bx}^2.
 \end{align*}
 \end{lemma}
\begin{proof}
    Fix $\bx\in\com$. From the definition of the Euclidean projection we have $\norm{\bx_{t+1}-\bx}^2 \leq \norm{\bx_t -  \eta_t\bg_t-\bx}^2$. Equivalently, 
    \begin{align*}
        \langle \bg_t , \bx_t-\bx\rangle\leq (2\eta_t)^{-1}(\norm{\bx_t-\bx}^2 -\norm{\bx_{t+1}-\bx}^2 )+\frac{\eta_t}{2}\norm{\bg_t}^2.
    \end{align*}
Let $a_t = \norm{\bx_t-\bx}^2$. Since $f$ is an $\alpha$-strongly convex function we have
 \begin{align*}
     f(\bx_t) -f(\bx)&\leq \langle\nabla f(\bx_t),\bx_t-\bx\rangle - \frac{\alpha}{2}a_t
     \\&=\langle\bg_t,\bx_t-\bx\rangle+\langle\nabla f(\bx_t)-\bg_t,\bx_t-\bx\rangle - \frac{\alpha}{2}a_t
     \\&\leq 
    (2\eta_t)^{-1}(a_t -a_{t+1}) +\langle\nabla f(\bx_t)-\bg_t,\bx_t-\bx\rangle+\frac{\eta_t}{2}\norm{\bg_t}^2-\frac{\alpha}{2}a_t.
 \end{align*}
 Taking the conditional expectation given $\bx_t$ and using the inequality $ab\leq a^2/\lambda + \lambda b^2/4$ valid for all $a,b\in \mathbb{R}$ and $\lambda>0$ we deduce that
  \begin{align*}
     f(\bx_t) -f(\bx)&\leq 
     (2\eta_t)^{-1}(a_t -\Exp\left[a_{t+1}|\bx_t\right] )+\langle\nabla f(\bx_t)-\Exp\left[\bg_t|\bx_t\right],\bx_t-\bx\rangle+\frac{\eta_t}{2}\Exp\left[\norm{\bg_t}^2|\bx_t\right]-\frac{\alpha}{2}a_t
     \\&\leq  (2\eta_t)^{-1}(a_t -\Exp\left[a_{t+1}|\bx_t\right] )+\norm{\nabla f(\bx_t)-\Exp\left[\bg_t|\bx_t\right]}\norm{\bx_t-\bx}+\frac{\eta_t}{2}\Exp\left[\norm{\bg_t}^2|\bx_t\right]-\frac{\alpha}{2}a_t
     \\&\leq (2\eta_t)^{-1}(a_t -\Exp\left[a_{t+1}|\bx_t\right] ) +\frac{1}{\alpha}\norm{\nabla f(\bx_t)-\Exp\left[\bg_t|\bx_t\right]}^2+\frac{\eta_t}{2}\Exp\left[\norm{\bg_t}^2|\bx_t\right]-\frac{\alpha}{4}a_t.
 \end{align*}
 \end{proof}

{\it Proof of Theorem \ref{thm1}.}
%
Since $\eta_t\leq 1/(18\bL{\color{black}d} \kappa)$ from Lemma \ref{lem:Aux2} we have 
\begin{align*}
       \Exp\left[f(\bx_{t+1})|\bx_t\right] 
    &\leq f(\bx_t) -\frac{\eta_t}{4}\norm{\nabla f(\bx_t)}^2+\frac{\eta_t}{2}d\left(L\kappa_{\beta}h_t^{\beta-1}\right)^2+\frac{3\bar{L} \eta_t^2}{4}\kappa d\left(\frac{\sigma^2}{h_t^2} + \frac{3\bar{L}^2{\color{black}d}}{4}h_t^2\right).
\end{align*}
Taking the expectation from both sides of this inequality and using the fact that $f$ is an $\alpha$-PL function we obtain
\begin{align}\label{eq:thm1}
       \delta_{t+1}
    &\leq \delta_t\left(1 - \frac{\eta_t\alpha}{2}\right)+\frac{\eta_t}{2}d\left(L\kappa_{\beta}h_t^{\beta-1}\right)^2+\frac{3\bar{L} \eta_t^2}{4}\kappa d\left(\frac{\sigma^2}{h_t^2} + \frac{3\bar{L}^2{\color{black}d}}{4}h_t^2\right),
\end{align}
where $\delta_t = \Exp\left[f(\bx_{t})-f^*\right] $. Let $T_0 = \floor{72\bL{\color{black}d}\kappa/\alpha}$. Note that 
\begin{align*}
    \eta_t = \begin{cases} \frac{1}{18\bL{\color{black}d}\kappa}\quad\quad &\text{if $t{\le } T_0$},\vspace{3mm}\\
    \frac{4}{\alpha t} \quad\quad &\text{if $t\geq T_0+1$},
    \end{cases}
\end{align*}
and 
\begin{align}\label{eq:etat}
    \frac{4}{(T_0+1)\alpha}\leq \eta_t\leq \frac{4}{T_0\alpha}, \quad \text{for} \ t\le T_0.
\end{align}

We consider separately the cases $T\ge T_0+1$ and $T\le T_0$. 

{If $T\ge T_0+1$ and $T\in \{1,2\}$ then, using \eqref{eq:thm1}, it is easy to check that the result of the theorem holds true. 

Consider now the case $T\ge 3$  and $T\ge T_0+1$, that is, $T\ge T_0'+1$, where $T_0'=\max(T_0,2)$.} For any $t \geq T_0+1$ we have $\eta_t = 4/(\alpha t)$ and we can write
\begin{align*}
       \delta_{t+1}
    &\leq \delta_t\left(1 - \frac{2}{t}\right)+\frac{d}{\alpha t}\left(2\left(L\kappa_{\beta}h_t^{\beta-1}\right)^2+\frac{3\bar{L}}{\alpha t}\kappa \left(\frac{\sigma^2}{h_t^2} + \frac{3\bar{L}^2{\color{black}d}}{4}h_t^2\right)\right).
\end{align*}
Substituting here $h_t = (\frac{3\bL}{\alpha t}\cdot\frac{\kappa \sigma^2}{2L^2\kappa_{\beta}^2})^{\frac{1}{2\beta}}$ gives
\begin{align*}
     \delta_{t+1}
    &\leq \delta_{t} \left(1 - \frac{2}{t}\right)+\cst\frac{d}{\alpha t}\left({\left(\frac{1}{\alpha t}\right)^{\frac{\beta-1}{\beta}}}+ {\frac{{\color{black}d}}{t}\left(\frac{1}{\alpha t}\right)^{\frac{1}{\beta}}}\right),
\end{align*}
where {$\cst>0$ depends only on ${L, \bL}$, $\beta$, and $\sigma^2$, and its value may vary across different appearances throughout this proof.}
Since $T\geq {T_0'+1 \ge 3}$, applying \cite[Lemma 32]{akhavan2023gradient} leads to the bound
\begin{align}\label{eq:deltat}
      \delta_T
    &\leq \frac{2{T_0'}}{T}\delta_{{T_0'}+1} +
    \cst\frac{d}{\alpha }\left({\left(\frac{1}{\alpha T}\right)^{\frac{\beta-1}{\beta}}}+ {\frac{{\color{black}d}}{T}\left(\frac{1}{\alpha T }\right)^{\frac{1}{\beta}}}\right).
\end{align}
{If $T_0\le 2$ (that is, $T_0'=2$) we use  \eqref{eq:thm1} with $t=1, t=2,$ the definition of $T_0$, \eqref{eq:etat}, and the definition of $h_t$ to obtain}
{
$$
\delta_{{T_0'}+1} = \delta_3 \le \delta_1 + \cst\frac{d}{\alpha }\left(\left(\frac{1}{\alpha }\right)^{\frac{\beta-1}{\beta}}+ {d}\left(\frac{1}{\alpha  }\right)^{\frac{1}{\beta}}\right).
$$
}
{
Combining this bound with \eqref{eq:deltat} completes the proof for the case $T_0'= 2$, $T\ge T_0'+1$.

Let now $T_0'\ge 3$ (which implies $T_0=T_0'$), and $T\ge T_0'+1$. }
From \eqref{eq:thm1}, \eqref{eq:etat} and the definition of $h_t$ we obtain that, for any $t\le T_0$,
{
\begin{align}\label{eq:deltat1}
    \delta_{t+1}\leq \delta_t\left(1-\frac{2}{T_0+1}\right) +\cst \frac{d}{\alpha T_0}\left(\left(\frac{1}{\alpha}\right)^{\frac{\beta-1}{\beta}}\left(T^{-\frac{\beta-1}{\beta}} + \frac{T^{\frac{1}{\beta}} }{T_0}\right) + \frac{{\color{black}d}}{T_0}\left(\frac{1}{\alpha T}\right)^{\frac{1}{\beta}}\right).
\end{align}
}
By iterating \eqref{eq:deltat1} with the simplification $1-2/(T_0+1) \leq 1$, and taking into account the definition of $T_0$ and the inequality $T \ge  T_0+1$ we find that
{
\begin{align*}
    \frac{2T_0}{T}\delta_{{T_0}+1}\leq \frac{144\bL {\color{black}d}\kappa}{\alpha T}\delta_1+ \cst\frac{{}d}{\alpha }\left(\left(\frac{1}{\alpha T}\right)^{\frac{\beta-1}{\beta}} + \frac{{\color{black}d}}{T}\left(\frac{1}{\alpha T}\right)^{\frac{1}{\beta}}\right).
\end{align*}
}
This bound and \eqref{eq:deltat} imply that
{
\begin{align*}
    \delta_T \leq \frac{144\bL {\color{black}d}\kappa}{\alpha T}\delta_1 +   \cst\frac{d}{\alpha }\left(\left(\frac{1}{\alpha T}\right)^{\frac{\beta-1}{\beta}}+ \frac{{\color{black}d}}{T}\left(\frac{1}{\alpha T }\right)^{\frac{1}{\beta}}\right).
\end{align*}
}
The proof for the case $T\ge T_0+1$ is complete.

Consider now the case $T \leq T_0$. {In this case \eqref{eq:deltat1} holds for all $t\le T$ and we have the bound}  
{
\begin{align*}
 \delta_{T+1}&\leq 
    \delta_1\left(1-\frac{2}{T_0+1}\right)^T + \cst\frac{d}{\alpha }\left(\left(\frac{1}{\alpha T}\right)^{\frac{\beta-1}{\beta}} + \frac{{\color{black}d}}{T}\left(\frac{1}{\alpha T}\right)^{\frac{1}{\beta}}\right).
\end{align*}
}
Since $(1-\lambda)^T \leq \exp(-\lambda T) \leq 1/(\lambda T)$ for all {$T>0,\lambda \in (0,1)$}
we obtain
{
\begin{align*}
 \delta_{T+1}&\leq 
    \frac{T_0 + 1}{2T}\delta_1 + \cst\frac{d}{\alpha }\left(\left(\frac{1}{\alpha T}\right)^{\frac{\beta-1}{\beta}} + \frac{{\color{black}d}}{T}\left(\frac{1}{\alpha T}\right)^{\frac{1}{\beta}}\right).
\end{align*}
}
{
By using the fact that {$4TT_0\ge (T+1)(T_0+1)$} we finally get
\begin{align*}
 \delta_{T+1}&\leq 
    \frac{2T_0}{T+1}\delta_1 + \cst\frac{d}{\alpha }\left(\left(\frac{1}{\alpha (T+1)}\right)^{\frac{\beta-1}{\beta}} + \frac{{\color{black}d}}{T+1}\left(\frac{1}{\alpha (T+1)}\right)^{\frac{1}{\beta}}\right)
    \\&\leq \frac{144\bL{\color{black}d}\kappa}{\alpha (T+1)}\delta_1 + \cst\frac{d}{\alpha }\left(\left(\frac{1}{\alpha (T+1)}\right)^{\frac{\beta-1}{\beta}} + \frac{{\color{black}d}}{T+1}\left(\frac{1}{\alpha (T+1)}\right)^{\frac{1}{\beta}}\right).
\end{align*}
}
\vspace{2mm}

{\it Proof of Theorem \ref{thm2}.}
%
Fix $\bx\in\com$. By Lemma \ref{lem:Auxstr1} we have 
      \begin{align*}
     f(\bx_t) -f(\bx)\leq \frac{a_t -\Exp\left[a_{t+1}|\bx_t\right] }{2\eta_t} +\frac{1}{\alpha}\norm{\nabla f(\bx_t)-\Exp\left[\bg_t|\bx_t\right]}^2+\frac{\eta_t}{2}\Exp\left[\norm{\bg_t}^2|\bx_t\right]-\frac{\alpha}{4}a_t,
 \end{align*}
 where $a_t = \norm{\bx_t-\bx}^2$.
 Using Lemmas \ref{lem:bias} and \ref{lem:var} and the assumption that $\max_{\bx\in\com}\norm{\nabla f(\bx)}\leq G$ we obtain 
\begin{align*}
f(\bx_t) -f(\bx)&\leq \frac{a_t -\Exp\left[a_{t+1}|\bx_t\right] }{2\eta_t} +\frac{d}{\alpha}(\kappa_{\beta}Lh_t^{\beta-1})^2
+\\ 
&\quad +\frac{3\eta_t}{4}\kappa d\left(\frac{3}{4}\left(d\bL^2 h_t^2 + 8G^2\right) + \frac{\sigma^2}{h_t^2}\right)-\frac{\alpha a_t}{4}.
\end{align*}
Let $b_t = \Exp\left[\norm{\bx_t-\bx}^2\right]$. Substituting here $\eta_t = 4/(\alpha (t+1))$ and taking the expectation we obtain 
\begin{align*}
\Exp\left[f(\bx_t) -f(\bx)\right]&\leq \frac{\alpha}{8}\left((t+1)\left(b_t-b_{t+1}\right)-2b_t\right)
+\\ 
&\quad+\frac{d}{\alpha}(\kappa_{\beta}Lh_t^{\beta-1})^2+\frac{3}{\alpha (t+1)}\kappa d\left(\frac{3}{4}\left(d\bL^2 h_t^2 + 8G^2\right) + \frac{\sigma^2}{h_t^2}\right).
\end{align*}
Summing up both sides of this inequality from $1$ to $T$ and using the fact that
\begin{align*}
\sum_{t=1}^{T}t\left((t+1)\left(b_t-b_{t+1}\right)-2b_t\right)\leq 0,
\end{align*}
we find
\begin{align*}
&\sum_{t=1}^{T} t\Exp\left[f(\bx_t) -f(\bx)\right]\leq 
\\ 
&\qquad\leq \frac{d}{\alpha}\sum_{t=1}^{T}\left(t(\kappa_{\beta}Lh_t^{\beta-1})^2+\frac{3t}{t+1}\kappa \left(\frac{3}{4}\left(d\bL^2 h_t^2 + 8G^2\right) + \frac{\sigma^2}{h_t^2}\right)\right).
\end{align*}
Substituting here $h_t = (\frac{3}{2t}\frac{\kappa\sigma^2}{\kappa_\beta^2L^2})^{\frac{1}{2\beta}}$ we get the inequality
\begin{align*}
\sum_{t=1}^{T} t\Exp\left[f(\bx_t) -f(\bx)\right]&\leq  \frac{18G^2d\kappa T}{\alpha}+\cst \frac{d}{\alpha}\sum_{t=1}^{T}\left( t^{\frac{1}{\beta}}+d t^{-\frac{1}{\beta}}\right)\le
\\
&\quad \le
\frac{18G^2d\kappa T}{\alpha}+\cst' \frac{d}{\alpha}\left( 1+dT^{-\frac{2}{\beta}}\right)T^{\frac{\beta+1}{\beta}}.
\end{align*}
 Here $\cst, \cst'$ are positive constants depending only on $L$, $\bL$, $\beta$, and $\sigma^2$. Dividing both sides of this inequality by $T(T+1)/2$ and applying Jensen's inequality we get the result of the theorem.

\section*{Acknowledgement}
The research of Arya Akhavan was funded by UK Research and Innovation (UKRI) under the UK government’s Horizon Europe funding guarantee [grant number EP/Y028333/1].

\bibliographystyle{abbrv}
\bibliography{bibliography.bib}
\end{document}